%% file: kkks11.tex
\documentclass[10pt]{article}
\pdfoutput=1

\usepackage{graphicx} % more modern
\usepackage{subfigure}
\usepackage[margin=1in]{geometry}

\usepackage{algorithm}
\usepackage{algorithmic}

\usepackage{amsmath, amssymb, amsthm}
\usepackage{color}

\def\ball{\mathbb{B}}
\def\reals{\mathbb{R}}
\def\LPAV{{\sf LPAV}}
\def\PAV{{\sf PAV}}
\def\GLMt{\textsf{GLM-tron}}
\def\Liso{\textsf{L-Isotron}}

\def\err{\mathrm{err}}
\def\E{\mathbb{E}}
\def\herr{\widehat{\mathrm{err}}}
\def\eps{\varepsilon}
\def\heps{\hat{\varepsilon}}

\def\fraconem{\frac{1}{m}}
\def\hatyti{\hat{y}^t_i}
\def\sumionetom{\sum_{i=1}^m}

\def\baryi{\bar{y}_i}
\def\tildeyti{\tilde{y}^t_i}

\newtheorem{proposition}{Proposition}
\newcommand{\norm}[1]{\left\|#1\right\|}
\newcommand{\inner}[1]{\langle#1\rangle}
\newcommand{\Ucal}{\mathcal{U}}
\newcommand{\Wcal}{\mathcal{W}}
\newcommand{\Fcal}{\mathcal{F}}
\newcommand{\Ncal}{\mathcal{N}}
\newcommand{\lemref}[1]{Lemma~\ref{#1}}
\newcommand{\thmref}[1]{Thm.~\ref{#1}}

\newcommand{\lnorm}[1]{\ensuremath \left\| #1 \right\|}

\newtheorem{lemma}{Lemma}
\newtheorem{thm}{Theorem}

\begin{document}

%\twocolumn[
%\icmltitle{Efficient Learning of Generalized Linear and Single Index Models with Isotonic Regression}
\title{Efficient Learning of Generalized Linear and Single Index Models with Isotonic Regression}
%\icmltitle{Efficient Generalized Linear Modelling with\\ High-Dimensional Isotonic Regression}

% unless you've provided the [accepted] option to the icml2011
\author{
Sham Kakade \\ 
\texttt{skakade@wharton.upenn.edu} \\ 
Wharton, University of Pennsylvania \\ 
Philadelphia, PA 19104 USA
\and
Adam Tauman Kalai \\
\texttt{adum@microsoft.com} \\
Microsoft Research \\
Cambridge, MA 02139 USA 
\and 
Varun Kanade\thanks{This work was partially supported by grants NSF-CCF-04-27129
and NSF-CCF-09-64401.}\\
\texttt{vkanade@fas.harvard.edu} \\
Harvard University \\
Cambridge MA 02138 
\and
Ohad Shamir \\
\texttt{ohadsh@microsoft.com} \\
Microsoft Research \\
Cambridge MA 02139
}
% \icmlauthor{Sham Kakade}{skakade@wharton.upenn.edu}
% \icmladdress{Wharton, University of Pennsylvania,
%             Philadelphia, PA 19104 USA}
% \icmlauthor{Adam Tauman Kalai}{adum@microsoft.com}
% \icmladdress{Microsoft Research,
%             Cambridge, MA 02139 USA}
% \icmlauthor{Varun Kanade}{vkanade@fas.harvard.edu}
% \icmladdress{Harvard University,
%             Cambridge, MA 02138 USA}
% \icmlauthor{Ohad Shamir}{ohadsh@microsoft.com}
% \icmladdress{Microsoft Research,
%             Cambridge, MA 02138 USA}
% 
% % You may provide any keywords that you
% % find helpful for describing your paper; these are used to populate
% % the "keywords" metadata in the PDF but will not be shown in the document
% \icmlkeywords{learning theory, isotonic regression, generalized linear models,
% single index models}
% 
% \vskip 0.3in
% ]
\maketitle

\begin{abstract}
  Generalized Linear Models (GLMs) and Single Index Models (SIMs)
  provide powerful generalizations of linear regression, where the
  target variable is assumed to be a (possibly unknown) 1-dimensional
  function of a linear predictor. In general, these problems entail
  non-convex estimation procedures, and, in practice, iterative
  local search heuristics are often used.
  %where \emph{efficiently} obtaining
 % a provably good predictor is a non-trivial.
  Kalai and Sastry (2009) recently provided the first provably efficient method
  for learning SIMs and GLMs, under the assumptions that the data are
  in fact generated under a GLM and under certain monotonicity and
  Lipschitz constraints. However, to obtain provable performance, the
  method requires a fresh sample every iteration. In this paper, we
  provide algorithms for learning GLMs and SIMs, which are both
  computationally and statistically efficient. We also provide an
  empirical study, demonstrating their feasibility in practice.

\end{abstract}

\section{Introduction}
\label{sec:intro}

\input{intro}

\section{Setting}

\input{setting}

\section{The $\GLMt$ Algorithm}\label{sec:GLM}

\input{glm_algo}

\section{The $\Liso$ Algorithm}

\input{L_Isotron}

\section{Proofs}\label{sec:proofs}

\input{proofs}

\section{Experiments}
\label{sec:expt}

\input{experiments}

\bibliography{kkks11}
\bibliographystyle{alpha}
%%%%%%%%%%%%%%%%%%%%%%%%%%%%%%%%%%%%%%%%%%%%%%%%%%%%%%%%
%%%%%%%%%%%%%%%%%%%%%%%%%%%%%%%%%%%%%%%%%%%%%%%%%%%%%%%%
\appendix
\clearpage
%\newpage
%\onecolumn

\section{Appendix}

\input{appendix}

\end{document}

%% file: intro.tex
%regression to GLMs
The oft used linear regression paradigm models a target variable $Y$ as a linear
function of a vector-valued input $X$. Namely, for some vector $w$, we assume
that $\mathbb{E}[Y|X]=w\cdot X$. Generalized linear models (GLMs) provide a
flexible extension of linear regression, by assuming the existence of a ``link''
function $g$ such that $\E[Y|X] = g^{-1}(w\cdot X)$. $g$ ``links'' the
conditional expectation of $Y$ to $X$ in a linear manner, i.e.  $g(\E[Y|X])=
w\cdot X$ (see \cite{mn89:glm2e} for a review). This simple assumption
immediately leads to many practical models, including logistic regression, the
workhorse for binary probabilistic modeling.

%learning the link function? stat vs. computation?
Typically, the link function is assumed to be known (often chosen based on
problem-specific constraints), and the parameter $w$ is estimated using some
iterative procedure.  Even in the setting where $g$ is known, we are not aware
of a classical estimation procedure which is computationally efficient, yet
achieves a good statistical rate with provable guarantees. The standard
procedure is iteratively reweighted least squares, based on Newton-Ralphson
(see~\cite{mn89:glm2e}).

In Single Index Models (SIMs), \emph{both} $g$ and $w$ are unknown. Here, we
face the more challenging (and practically relevant) question of jointly
estimating $g$ and $w$, where $g$ may come from a large non-parametric family
such as all monotonic functions. There are two issues here: 1) What statistical
rate is achievable for simultaneous estimation of $g$ and $w$? 2) Is there a
\emph{computationally} efficient algorithm for this joint estimation? With
regards to the former, under mild Lipschitz-continuity restrictions on $g^{-1}$,
it is possible to characterize the effectiveness of an (appropriately
constrained) joint empirical risk minimization procedure. This suggests that,
from a purely statistical viewpoint, it may be worthwhile to attempt to jointly
optimize $g$ and $w$ on the empirical data.

However, the issue of \emph{computationally efficiently} estimating both $g$ and
$w$ (and still achieving a good statistical rate) is more delicate, and is the
focus of this work. We note that this is not a trivial problem: in general, the
joint estimation problem is highly non-convex, and despite a significant body of
literature on the problem, existing methods are usually based on heuristics,
which are not guaranteed to converge to a global optimum (see for instance
\cite{HarHaIch93,HoHar94,HriJuSpo98,NaTsa04,RaviWaYu08}). We note that recently,
\cite{SSShSr10} presented a kernel-based method which does allow (improper)
learning of certain types of GLM's and SIM's, even in an agnostic setting where
no assumptions are made on the underlying distribution. On the flip side, the
formal computational complexity guarantee degrades super-polynomially with the
norm of $w$, which \cite{SSShSr10} show is provably unavoidable in their
setting.

The recently proposed Isotron algorithm \cite{KS09} provides the first provably
efficient method for learning GLMs and SIMs, under the common assumption that
$g^{-1}$ is  monotonic and Lipschitz, and assuming the data corresponds to the
model. The algorithm attained both polynomial sample and computational
complexity, with a sample size dependence that does not depend explicitly on the
dimension. The algorithm is a variant of the ``gradient-like'' perceptron
algorithm, with the added twist that on each update, an isotonic regression
procedure is performed on the linear predictions. Recall that isotonic
regression is a procedure which finds the best monotonic one dimensional
regression function. Here, the well-known Pool Adjacent Violator ($\PAV$)
algorithm provides a computationally efficient method for this task.

Unfortunately, a cursory inspection of the Isotron algorithm suggests
that, while it is computationally efficient, it is very wasteful
statistically, 
\iffalse
to the point of being non-practical, 
\fi
as each iteration of the algorithm throws away all previous training data and
requests new examples. Our intuition is that the underlying technical reasons
for this 
%wastefulness
are due to the fact that the $\PAV$ algorithm need not return a function with a
bounded Lipschitz constant.  Furthermore, empirically, it not clear how
deleterious this issue may be.
\iffalse
 We
have two conjectures here: 1) the assumption that data are generated
according to a true GLM is vital (and that a hardness result is
possible without this assumption and 2) the theoretical rate provided
by the Isotron algorithm is not provable on. Formalizing these
\fi

This work seeks to address these issues both theoretically and practically. We
present two algorithms, the $\GLMt$ algorithm for learning GLMs with a known
monotonic and Lipschitz $g^{-1}$, and the $\Liso$ algorithm for the more general
problem of learning SIMs, with an unknown monotonic and Lipschitz $g^{-1}$. Both
algorithms are practical, parameter-free and are provably efficient, both
statistically and computationally. Moreover, they are both easily kernelizable.
In addition, we investigate both algorithms empirically, and show they are both
feasible approaches. Furthermore, our results show that the original Isotron
algorithm (ran on the same data each time) is perhaps also effective in several
cases, even though the $\PAV$ algorithm does not have a Lipschitz constraint.

More generally, it is interesting to note how the statistical assumption that
the data are in fact generated by some GLM leads to an efficient estimation
procedure, despite it being a non-convex problem. Without making any
assumptions, i.e. in the agnostic setting, this problem is at least hard as
learning parities with noise.

%% file: setting.tex
We assume the data $(x,y)$ are sampled i.i.d. from a
distribution supported on $\ball_d \times [0, 1]$, where $\ball_d=\{ x
\in \reals^d ~:~ \lnorm{x} \leq 1\}$ is the unit ball in
$d$-dimensional Euclidean space. Our algorithms and analysis also
apply to the case where $\ball_d$ is the unit ball in some high (or
infinite)-dimensional kernel feature space. We assume there is a fixed
vector $w$, such that $\lnorm{w}\leq W$, and a non-decreasing
$1$-Lipschitz function $u: \reals \rightarrow [0, 1]$, such that
$\E[y|x]=u(w\cdot x)$ for all $x$. Note that $u$ plays the same role
here as $g^{-1}$ in generalized linear models, and we use this notation for convenience. Also, the restriction that $u$ is 1-Lipschitz is without loss of generality, since the norm of $w$ is arbitrary (an equivalent restriction  is that $\lnorm{w}= 1$ and that $u$ is $W$-Lipschitz for an arbitrary
$W$).

Our focus is on approximating the regression function well, as
measured by the squared loss. For a real valued function $h : \ball_d
\rightarrow [0, 1]$, define
\begin{align*}
\err(h) &= \E_{(x,y)} \left[(h(x) - y)^2\right]\\
\eps(h) &= \err(h) - \err(E[y|x])\\
& = \E_{(x,y)} \left[(h(x) - u(w\cdot x))^2\right]
%\E_{(x,y)} \left[(h(x) - y)^2-(u(w\cdot x)-y)^2\right].
\end{align*}
$\err(h)$ measures the error of $h$, and $\eps(h)$ measures the excess
error of $h$ compared to the Bayes-optimal predictor $x\mapsto u(w\cdot
x)$. Our goal is to find $h$ such that $\eps(h)$ (equivalently,
$\err(h)$) is as small as possible.

In addition, we define the empirical counterpart $\herr(h),\heps(h)$,
based on a sample $(x_1,y_1),\ldots,(x_m,y_m)$, to be
\begin{align*}
\herr(h) &= \fraconem \sumionetom (h(x_i)-y_i)^2\\
\heps(h) &= \fraconem \sumionetom (h(x_i) - u(w\cdot x_i))^2.
%\fraconem \sumionetom ((h(x_i)-y_i)^2 - (u(w\cdot x_i)-y_i)^2).
\end{align*}
Note that $\heps$ is the standard \emph{fixed design error} (as this
error conditions on the observed $x$'s).

\iffalse
Note that $\heps(h)$ can also be written as
\[
\heps(h) =
\]
\fi

Our algorithms work by iteratively constructing hypotheses $h^t$ of
the form $h^t(x)=u^t(w^t \cdot x)$, where $u^t$ is a non-decreasing,
$1$-Lipschitz function, and $w^t$ is a linear predictor. The
algorithmic analysis provides conditions under which $\heps(h^t)$ is
small, and using statistical arguments, one can guarantee that
$\eps(h^t)$ would be small as well.

To simplify the presentation of our results, we use the standard
$O(\cdot)$ notation, which always hides only universal constants.

%% file: glm_algo.tex
We begin with the simpler case, where the transfer function $u$ is assumed to be known (e.g. a sigmoid), and the problem is estimating $w$ properly. We present a simple, parameter-free, perceptron-like algorithm, $\GLMt$, which efficiently finds a close-to-optimal predictor. We note that the algorithm works for arbitrary non-decreasing, Lipschitz functions $u$, and thus covers most generalized linear models. The pseudo-code appears as Algorithm \ref{alg:fixed-u-alg}.

\begin{algorithm}[!t]
	\caption{\GLMt}
	\label{alg:fixed-u-alg}
\begin{algorithmic}
	\STATE {\bfseries Input:} data $\langle (x_i, y_i) \rangle_{i=1}^m \in
	\reals^d \times [0, 1]$, $u: \reals \rightarrow [0, 1]$.
	\STATE $w^1 := 0$;
	\FOR {$t = 1, 2, \ldots$}
	\STATE $h^t(x) := u(w^t \cdot x)$;
	\STATE $w^{t+1} := w^t + \displaystyle\fraconem \sumionetom (y_i - u(w^t
	\cdot x_i)) x_i$;
	\ENDFOR
\end{algorithmic}
\end{algorithm}

To analyze the performance of the algorithm, we show that if we run the algorithm for sufficiently many iterations, one of the predictors $h^t$ obtained must be nearly-optimal, compared to the Bayes-optimal predictor.

\begin{thm} \label{thm:known-u}
Suppose $(x_1,y_1),\ldots,(x_m,y_m)$ are drawn independently from a distribution supported on $\ball_d \times [0,1]$, such that $\E[ y | x] = u (w \cdot x)$, where $\lnorm{w} \leq W$, and $u:
\reals \rightarrow [0, 1]$ is a known non-decreasing $1$-Lipschitz function.
Then for any $\delta\in (0,1)$, the following holds with probability at least $1-\delta$: there exists some iteration $t < O(W \sqrt{m/\log(1/\delta)})$ of $\GLMt$ such
that the hypothesis $h^t(x)=u(w^t \cdot x)$ satisfies
\[
\max\{\heps(h^t),\eps(h^t)\} \leq O\left(\sqrt{\frac{W^2\log(m/\delta)}{m}}\right).
\]
\end{thm}
In particular, the theorem implies that some $h^t$ has $\eps(h^t)=O(1/\sqrt{m})$. Since $\eps(h^t)$ equals $\err(h^t)$ up to a constant, we can easily find an appropriate $h^t$ by using a hold-out set to estimate $\err(h^t)$, and picking the one with the lowest value.

The proof is along similar lines (but somewhat simpler) than the proof of our
subsequent \thmref{thm:main-thm}. The rough idea of the proof is showing  that
at each iteration, if $\heps(h^t)$ is not small, then the squared distance
$\lnorm{w^{t+1} - w^t}^2$ is substantially smaller than $\lnorm{w^t - w}^2$.
Since this is bounded below by $0$, and $\lnorm{w^0 - w}^2 \leq W^2$, there is
an iteration (arrived at within reasonable time) such that the hypothesis $h^t$
at that iteration is highly accurate. The proof is provided in Appendix
\ref{app:glm-algo}.

%% file: L_Isotron.tex
We now present $\Liso$, in Algorithm \ref{alg:main-alg}, which is applicable to the harder setting where the transfer function $u$ is unknown, except for it being non-decreasing and $1$-Lipschitz. This corresponds to the semi-parametric setting of single index models.

The algorithm that we present is again simple and parameter-free. The main difference compared to $\GLMt$ algorithm is that now the transfer function must also be learned, and the algorithm keeps track of a transfer function $u^t$ which changes from iteration to iteration. The algorithm is also rather similar to the Isotron algorithm \cite{KS09}, with the main difference being that instead of applying the $\PAV$ procedure to fit an arbitrary monotonic function at each iteration, we use a different procedure, $\LPAV$, which fits a \emph{Lipschitz} monotonic function. This difference is the key which allows us to make the algorithm practical while maintaining non-trivial guarantees (getting similar guarantees for the Isotron required a fresh training sample at each iteration).

The $\LPAV$ procedure takes as input a set of points $(z_1,y_1),\ldots,(z_m,y_m)$ in $\reals^2$, and fits a non-decreasing, $1$-Lipschitz function $u$, which minimizes $\sum_{i=1}^{m}(u(z_i)-y_i)^2$. This problem has been studied in the literature, and we followed the method of \cite{YW09} in our empirical studies. The running time of the method proposed in \cite{YW09} is $O(m^2)$. While this can be slow for large-scale datasets, we remind the reader that this is a one-dimensional fitting problem, and thus a highly accurate fit can be achieved by randomly subsampling the data (the details of this argument, while straightforward, are beyond the scope of the paper).

\begin{algorithm}[!t]
   \caption{$\Liso$}
	\label{alg:main-alg}
\begin{algorithmic}
   \STATE {\bfseries Input:} data $\langle (x_i, y_i) \rangle_{i=1}^m \in
	\reals^d \times [0, 1]$.
	\STATE $w^1 := 0$;
	\FOR {$t = 1, 2, \ldots$}
	\STATE $u^t := \LPAV\left((w^t \cdot x_1, y_1), \ldots, (w^t \cdot x_m , y_m)
	\right)$
	\STATE $w^{t+1} := w^t + \displaystyle\frac{1}{m} \sum_{i=1}^m (y_i - u^t(w^t \cdot x_i))
	x_i $
	\ENDFOR
\end{algorithmic}
\end{algorithm}

We now turn to the formal analysis of the algorithm. The formal guarantees parallel those of the previous subsection. However, the rates achieved are somewhat worse, due to the additional difficulty of simultaneously estimating both $u$ and $w$. It is plausible that these rates are sharp for information-theoretic reasons, based on the 1-dimensional lower bounds in \cite{Zhang02a} (although the assumptions are slightly different, and thus they do not directly apply to our setting).

\begin{thm}
\label{thm:main-thm}
Suppose $(x_1,y_1),\ldots,(x_m,y_m)$ are drawn independently from a distribution supported on $\ball_d \times [0,1]$, such that $\E[ y | x] = u (w \cdot x)$, where $\lnorm{w} \leq W$, and $u:
\reals \rightarrow [0, 1]$ is an unknown non-decreasing $1$-Lipschitz function. Then the following two bounds hold:
\begin{enumerate}
\item (Dimension-dependent) With probability at least $1 - \delta$, there exists some iteration $t < O\left(\left(\frac{Wm}{d \log(Wm/\delta)} \right)^{1/3}\right)$ of $\Liso$ such that
\[
\max\{\heps(h^t),\eps(h^t)\} \leq O\left(\left(\frac{dW^2 \log(Wm/\delta)}{m}\right)^{1/3}\right).
\]
\item (Dimension-independent) With probability at least $1 - \delta$, there exists some iteration $t
< O\left(\left( \frac{Wm}{\log(m/\delta)}\right)^{1/4}\right)$ of $\Liso$ such that
\[
\max\{\heps(h^t),\eps(h^t)\} \leq O \left(\left(\frac{W^2 \log(m/\delta)}{m}\right)^{1/4}\right)
\]
\end{enumerate}
\end{thm}

As in the case of \thmref{thm:known-u}, one can easily find $h^t$ which satisfies the theorem's conditions, by running the $\Liso$ algorithm for sufficiently many iterations, and choosing the hypothesis $h^t$ which minimizes $\err(h^t)$ based on a hold-out set.

%% file: proofs.tex
\subsection{Proof of \thmref{thm:main-thm}}

First we need a property of the $\LPAV$ algorithm that is used to find the best one-dimensional non-decreasing $1$-Lipschitz function. Formally, this problem can be defined as follows:
Given as input $\langle \{z_i, y_i\} \rangle_{i=1}^m \in [-W, W] \times [0, 1]$ the
goal is to find $\hat{y}_1,\ldots,\hat{y}_m$ such that
\begin{align}
\frac{1}{m} \sum_{i=1}^m (\hat{y}_i - y_i)^2, \label{eq:lpav-obj}
\end{align}
is minimal, under the constraint that $\hat{y}_i = u(z_i)$ for some non-decreasing $1$-Lipschitz function $u: [-W, W] \mapsto [0,1]$. After finding such values, $\LPAV$ obtains an entire function $u$ by interpolating linearly between the points. Assuming that $z_i$ are in sorted order, this can be
formulated as a quadratic problem with the following constraints:
\begin{align}
\hat{y}_i - \hat{y}_{i+1} &\leq 0 & 1 \leq i < m \label{eq:lpav-const-mon} \\
\hat{y}_{i+1} - \hat{y}_i - (z_{i+1} - z_{i}) & \leq 0 & 1 \leq i < m
\label{eq:lpav-const-lip}
\end{align}

\begin{lemma}
\label{lem:lpav-prop}
Let $(z_1, y_1),\ldots,(z_m,y_m)$ be input to $\LPAV$ where $z_i$ are
increasing and $y_i \in [0, 1]$. Let $\hat{y}_1, \ldots, \hat{y}_m$ be the
output of $\LPAV$. Let $f$ be any function such that $f(\beta) - f(\alpha) \geq
\beta - \alpha$, for $\beta \geq \alpha$, then
\[ \sumionetom (y_i - \hat{y}_i)(f(\hat{y}_i) - z_i) \geq 0 \]
\end{lemma}
\begin{proof}
We first note that $\sum_{j=1}^{m}(y_j-\hat{y}_j)=0$, since otherwise we could have found other values for $\hat{y}_1,\ldots,\hat{y}_m$ which make \eqref{eq:lpav-obj} even smaller. So for notational convenience, let $\hat{y}_0=0$, and we may assume w.l.o.g. that $f(\hat{y}_0)=0$. Define $\sigma_i = \sum_{j=i}^m (y_j - \hat{y}_j)$. Then we have
\begin{align}
&\sumionetom (y_i - \hat{y}_i)(f(\hat{y}_i) - z_i) =  \notag\\ &~~~~~~\sum_{i=1}^{m}
\sigma_{i} ((f(\hat{y}_{i}) - z_{i}) - (f(\hat{y}_{i-1}) - z_{i-1})).
\label{eq:decompo}
\end{align}
Suppose that $\sigma_i<0$. Intuitively, this means that if we could have decreased all values $\hat{y}_{i+1},\ldots \hat{y}_{m}$ by an infinitesimal constant, then the objective function \eqref{eq:lpav-obj} would have been reduced, contradicting the optimality of the values. This means that the constraint $\hat{y}_i-\hat{y}_{i+1}\leq 0$ must be tight, so we have
$(f(\hat{y}_{i+1}) - z_{i+1}) - (f(\hat{y}_i) - z_i) = -z_{i+1} + z_i \leq 0$
(this argument is informal, but can be easily formalized using KKT conditions).
Similarly, when $\sigma_i > 0$, then the constraint $\hat{y}_{i+1} - \hat{y}_i -
(z_{i+1} - z_{i})\leq 0$ must be tight, hence $f(\hat{y}_{i+1}) - f(\hat{y}_i)
\geq \hat{y}_{i+1} - \hat{y}_i = (z_{i+1} - z_i)\geq 0$. So in either case, each summand in \eqref{eq:decompo} must be non-negative, leading to the required result.
\end{proof}

We also use another result, for which we require a bit of
additional notation. At each iteration of the $\Liso$ algorithm, we run the $\LPAV$ procedure based on the training sample $(x_1,y_1),\ldots,(x_m,y_m)$ and the current direction $w^t$, and get a non-decreasing Lipschitz function $u^t$. Define
\[
\forall i~~\hatyti = u^t(w^t\cdot x_i).
\]
Recall that $w,u$ are such that $\E[y|x]=u(w\cdot x)$, and the input to the $\Liso$ algorithm is $(x_1,y_1),\ldots,(x_m,y_m)$. Define
\[
\forall i~~\bar{y}_i = u(w \cdot x_i)
\]
to be the expected value of each $y_i$. Clearly, we do not have access to $\bar{y}_i$. However, consider a hypothetical call to $\LPAV$ with inputs $\inner{ (w^t \cdot x_i, \baryi)}_{i=1}^m$, and suppose $\LPAV$ returns the function $\tilde{u}^t$. In that case, define
\[
\forall i~~\tildeyti = \tilde{u}^t(w^t\cdot x_i).
\]
for all $i$. Our proof uses the following proposition, which relates the values
$\hatyti$ (the values we can actually compute) and $\tildeyti$ (the values we
could compute if we had the conditional means of each $y_i$). The proof of
Proposition \ref{prop:close-fit} is somewhat lengthy and requires additional
technical machinery, and is therefore relegated to Appendix \ref{app:close-fit}.

\begin{proposition} \label{prop:close-fit}
With probability at least $1-\delta$ over the sample $\{(x_i, y_i)\}_{i=1}^m$, it holds for any $t$ that $\fraconem \sumionetom |\hatyti - \tildeyti|$ is at most the minimum of
\[
O\left(\left(\frac{dW^2\log(Wm/\delta)}{m}\right)^{1/3}\right)
\]
and
\[
O\left(\left(\frac{W^2 \log(m/\delta)}{m} \right)^{1/4}\right).
\]
\end{proposition}

The third auxiliary result we'll need is the following, which is well-known (see for example \cite{S-TCbook}, Section 4.1).
\begin{lemma}\label{lem:conc-bounds} Suppose $z_1, \ldots, z_m$ are i.i.d. $0$-mean random variables in a Hilbert space, such that $\Pr(\lnorm{x_i} \leq 1)=1$. Then with probability at least $1 - \delta$,
\[ \lnorm{\fraconem\sumionetom z_i} \leq 2 \left(\frac{1 +
\sqrt{\log(1/\delta)/2}}{\sqrt{m}}\right) \]
\end{lemma}

With these auxiliary results in hand, we can now turn to prove \thmref{thm:main-thm} itself. The heart of the proof is the following lemma, which shows that the squared distance $\lnorm{w^t -
w}^2$ between $w^t$ and the true direction $w$ decreases at each iteration at a rate which depends on the error of the hypothesis $\heps(h^t)$:

\begin{lemma}
\label{lem:update-step3}
Suppose that $\lnorm{w^t - w} \leq W$ and $\lnorm{(1/m) \sumionetom (y_i -
\bar{y}_i) x_i} \leq \eta_1$ and $(1/m) \sumionetom |\hatyti - \tildeyti| \leq
\eta_2$. Then
\[
\lnorm{w^t - w}^2 - \lnorm{w^{t+1} - w}^2 \geq \heps(h^t) - 5 W (\eta_1 +
\eta_2)
\]
\end{lemma}
\begin{proof}
We have
\begin{align}
&\lnorm{w^{t+1} - w}_2^2 = \lnorm{w^{t+1} - w^t + w^t - w}_2^2 \nonumber \\
&= \lnorm{w^{t+1} - w^t}_2^2 + \lnorm{w^t - w}_2^2  + 2(w^{t+1} - w^t) \cdot (w^t - w) \nonumber
\intertext{Since $w^{t+1} - w^t = (1/m)\sumionetom (y_i - \hatyti) x_i$,
substituting this above and rearranging the terms we get,}
&\lnorm{w^t - w}^2 - \lnorm{w^{t+1} - w}^2 \nonumber \\
&= \frac{2}{m} \sumionetom (y_i - \hatyti) (w \cdot x_i - w^t \cdot x_i)
- \lnorm{\fraconem\sumionetom(y_i - \hatyti) x_i}^2.
\label{eq:main-alg-main-eqn}
\end{align}
Consider the first term above,
\begin{align}
&\frac{2}{m} \sumionetom (y_i - \hatyti) (w \cdot x_i - w^t \cdot x_i) \nonumber \\
&= \left(\frac{2}{m} \sumionetom (y_i - \baryi)x_i \right) \cdot  (w - w^t)
\label{eq:main-alg-first-term1} \\
&~~~~~ + \frac{2}{m} \sumionetom (\baryi - \tildeyti)(w \cdot x_i - w^t \cdot x_i)
\label{eq:main-alg-first-term2} \\
&~~~~~ + \frac{2}{m} \sumionetom (\tildeyti - \hatyti) (w \cdot x_i - w^t \cdot
x_i) \label{eq:main-alg-first-term3}
\end{align}
The term (\ref{eq:main-alg-first-term1}) is at least $- 2W \eta_1$,
the term (\ref{eq:main-alg-first-term3}) is at least $-2W\eta_2$ (since $|(w -
w^t) \cdot x_i| \leq W$). We thus consider the remaining term
\eqref{eq:main-alg-first-term2}. Letting $u$ be the true transfer function, suppose for a minute it is strictly increasing, so its inverse $u^{-1}$ is well defined. Then we have
\begin{align*}
&\frac{2}{m} \sumionetom (\bar{y}_i - \tildeyti)(w \cdot x_i - w^t \cdot x_i)\\
&=
\frac{2}{m} \sumionetom (\bar{y}_i - \tildeyti) (w \cdot x_i -u^{-1}(\tildeyti)) \\
&~~~~ + \frac{2}{m} \sumionetom (\bar{y}_i -
\tildeyti)(u^{-1}(\tildeyti) - w^t \cdot x_i) \\
\end{align*}
The second term in the expression above is positive by Lemma \ref{lem:lpav-prop}. As to the first term, it is equal to $\frac{2}{m}\sum_{i=1}^{m}(\bar{y}_i - \tildeyti)(u^{-1}(\bar{y}_i) - u^{-1}(\tildeyti))$, which by the Lipschitz property of $u$ is at least $\frac{2}{m} \sumionetom (\bar{y}_i - \tildeyti)^2 = 2\heps(\tilde{h}^t)$. Plugging this in the above, we get
\begin{equation}\label{eq:main-alg-first-term}
\frac{2}{m} \sumionetom (y_i - \hatyti)(w \cdot x_i - w^t \cdot x_i) \geq
2\heps(\tilde{h}^t) - 2W (\eta_1 + \eta_2)
\end{equation}
This inequality was obtained under the assumption that $u$ is strictly increasing, but it is not hard to verify that the same would hold even if $u$ is only non-decreasing.

The second term in \eqref{eq:main-alg-main-eqn} can be bounded, using some tedious technical manipulations (see \eqref{eq:manipulation} and \eqref{eq:lem-known-u-second-term} in the supplementary material), by
\begin{equation}\label{eq:main-alg-second-term}
\lnorm{\fraconem\sumionetom (y_i - \hatyti) x_i}^2 \leq \heps(h^t) + 3 W
\eta_1
\end{equation}
Combining \eqref{eq:main-alg-first-term} and
\eqref{eq:main-alg-second-term}) in \eqref{eq:main-alg-main-eqn}, we get
\begin{equation}\label{eq:almostthere}
\lnorm{w^t - w}^2 - \lnorm{w^{t+1} - w}^2 \!\!\geq 2 \heps(\tilde{h}^t) -
\heps(h^t) - W(5\eta_1 + 2\eta_2)
\end{equation}
Now, we claim that
\[
\heps(\tilde{h}^t) - \heps(h^t) \geq  -\frac{2}{m} \sumionetom |\hatyti -
\tildeyti|\geq -2\eta_2,
\]
since
\begin{align*}
&\heps(\tilde{h}^t) = \fraconem \sumionetom (\tildeyti - \baryi)^2 \\
&~~~~= \fraconem \sumionetom (\tildeyti - \hatyti + \hatyti - \baryi)^2 \\
&~~~~= \fraconem \sumionetom (\hatyti - \baryi)^2 \\
&~~~~~~ + \left(\fraconem \sumionetom (\tildeyti - \hatyti) \right) (\tildeyti +
\hatyti - 2 \baryi)
\end{align*}
and we have that $|\tildeyti + \hatyti - 2\baryi| \leq 2$. Plugging this into \eqref{eq:almostthere} leads to the desired result.
\end{proof}

The bound on $\heps(h^t)$ in \thmref{thm:main-thm} now
follows from Lemma \ref{lem:update-step3}. Using the notation from Lemma \ref{lem:update-step3}, $\eta_1$ can be set to the bound in Lemma
\ref{lem:conc-bounds}, since $\{(y_i - \baryi)x_i\}_{i=1}^{m}$ are i.i.d. $0$-mean random variables with norm bounded by $1$. Also, $\eta_2$ can be set to any of the bounds in Proposition \ref{prop:close-fit}. $\eta_2$ is clearly the dominant term. Thus, we get that \lemref{lem:update-step3} holds, so either $\lnorm{w^{t+1} - w}^2\leq \lnorm{w^t - w}^2 -W(\eta_1+\eta_2)$, or  $\heps(h^t)\leq 3W(\eta_1+\eta_2)$. If the latter is the case, we are done. If
not, since $\lnorm{w^{t+1} - w}^2 \geq 0$, and $\lnorm{w^0 - w}^2 =
\lnorm{w}^2 \leq W^2$, there can be at most $W^2/(W(\eta_1+\eta_2)) = W/(\eta_1+\eta_2)$ iterations before $\heps(h^t) \leq 6 W \eta$. Plugging in the values for $\eta_1,\eta_2$ results in the bound on $\heps(h^t)$.

Finally, to get a bound on $\eps(h^t)$, we utilize the following uniform convergence lemma:
\begin{lemma}\label{lem:unconv}
    Suppose that $\E[y|x]=u(\inner{w,x})$ for some non-decreasing $1$-Lipschitz $u$ and $w$ such that $\norm{w}\leq W$. Then with probability at least $1-\delta$ over a sample $(x_1,y_1),\ldots,(x_m,y_m)$, the following holds simultaneously for any function $h(x)=\hat{u}(\hat{w}\cdot x)$ such that $\norm{\hat{w}}\leq W$ and a non-decreasing and $1$-Lipschitz function $\hat{u}$:
\[
\left|\eps(h)-\heps(h)\right|
\leq O\left(\sqrt{\frac{W^2\log(m/\delta)}{m}}\right).
\]
\end{lemma}
The proof of the lemma uses a covering number argument, and is shown as part of the more general \lemref{lem:dimfreebound} in the supplementary material. This lemma applies in particular to $h^t$. Combining this with the bound on $\heps(h^t)$, and using a union bound, we get the result on $\eps(h^t)$ as well.

%% file: experiments.tex
In this section, we present an empirical study of the $\GLMt$ and the $\Liso$ algorithms. The first experiment we performed is a synthetic one, and is meant to highlight the difference between $\Liso$ and the Isotron algorithm of \cite{KS09}. In particular, we show that attempting to fit the transfer function without any Lipschitz constraints may cause Isotron to overfit, complementing our theoretical findings. The second set of experiments is a comparison between $\GLMt$, $\Liso$ and several competing approaches. The goal of these experiments is to show that our algorithms perform well on real-world data, even when the distributional assumption required for their theoretical guarantees does not precisely hold.

\begin{figure}[ht]
%\centerline{
\begin{center}
\includegraphics[width=0.5\textwidth]{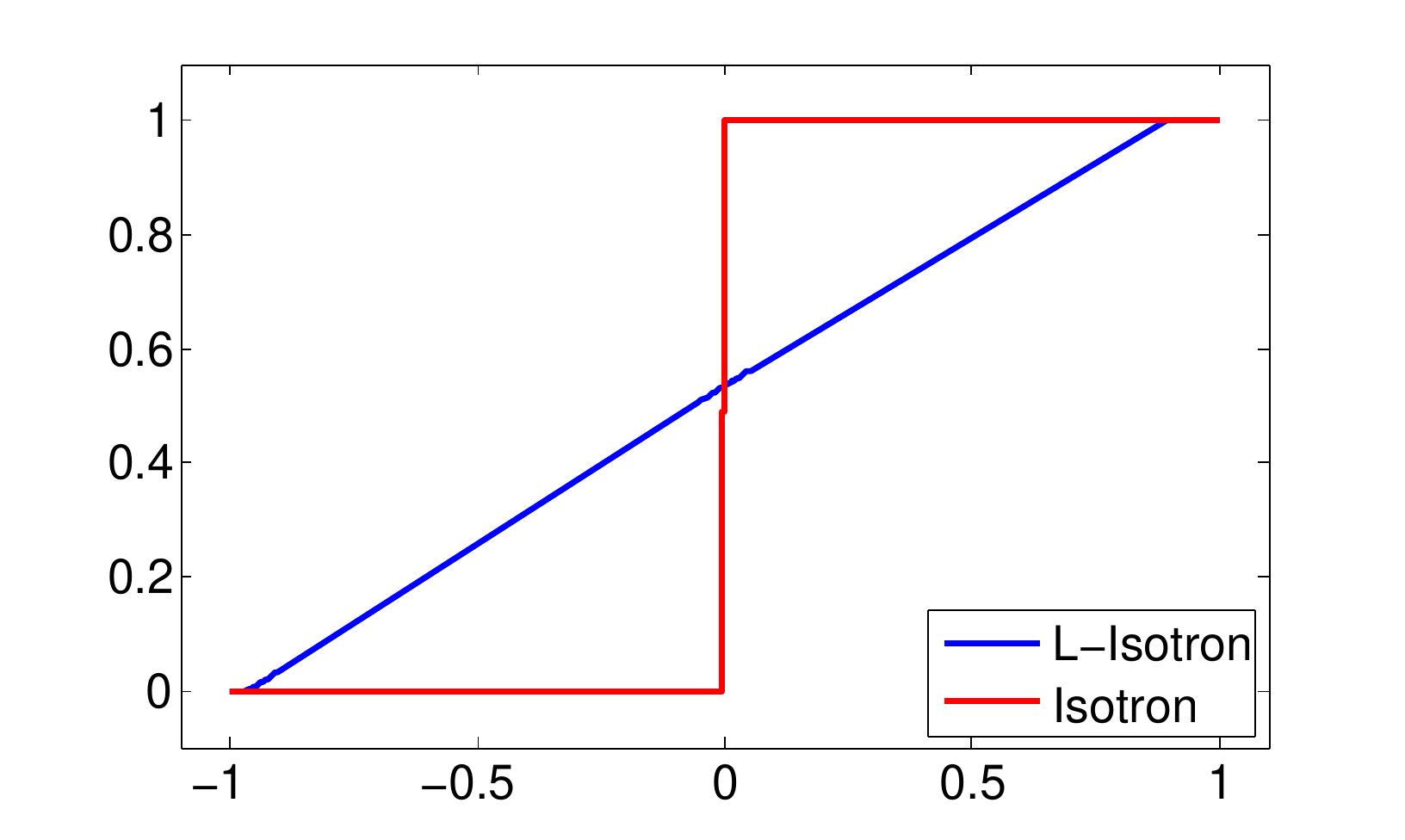}
\end{center}
%}
\caption{The link function as predicted by LIsotron (blue) and Isotron (red).
The domain of both functions was normalized to $[-1,1]$. \label{fig:overfit}}
\end{figure}

\subsection{$\Liso$ vs Isotron}
\label{sec:expt-synthetic}

As discussed earlier, our $\Liso$ algorithm (Algorithm \ref{alg:main-alg}) is
similar to the Isotron algorithm of \cite{KS09}, with two main differences:
First, we apply $\LPAV$ at each iteration to find the best \emph{Lipschitz}
monotonic function to the data, while they apply the $\PAV$ (Pool Adjacent
Violator) procedure to fit a monotonic (generally non-Lipschitz) function. The
second difference is the theoretical guarantees, which in the case of Isotron
required working with a fresh training sample at each iteration.

While the first difference is inherent, the second difference is just
an outcome of the analysis. In particular, one might still try and
apply the Isotron algorithm, using the same training sample at each
iteration. While we do not have theoretical guarantees for this
algorithm, it is computationally efficient, and one might wonder how
well it performs in practice. As we see later on, it actually performs
quite well on the datasets we examined. However, in this subsection we
provide a simple example, which shows that sometimes, the repeated
fitting of a \emph{non}-Lipschitz function, as done in the Isotron
algorithm, can cause overfit and thus hurt performance, compared to
fitting a Lipschitz function as done in the $\Liso$ algorithm.

We constructed a synthetic dataset as follows: In a high dimensional space
$(d = 400)$, we let $w = (1, 0, \ldots, 0)$ be the true direction. The transfer
function is $u(t) = (1 + t)/2$. Each data point $x$ is constructed as follows:
the first coordinate is chosen uniformly from the set $\{-1, 0, 1\}$, and out of
the remaining coordinates, one is chosen uniformly at random and is set to $1$.
All other coordinates are set to $0$. The $y$ values are chosen at random from
$\{0, 1\}$, so that $\E[ y | x] = u(w \cdot x)$. We used a sample of size $600$
to evaluate the performance of the algorithms.

In the synthetic example we construct, the first attribute is the only relevant attribute. However, because of the random noise in the $y$ values, Isotron tends to overfit using the irrelevant attributes. At data points where the true mean value $u(w \cdot x)$ equals $0.5$, Isotron (which uses $\PAV$) tries to fit the value $0$ or $1$, whichever is observed. On the other hand, $\Liso$ (which uses $\LPAV$) predicts this correctly as close to $0.5$, because of the Lipschitz constraint.  Figure \ref{fig:overfit} shows the link functions predicted by $\Liso$ and Isotron on this dataset. Repeating the experiment $10$ times, the error of $\Liso$, normalized by the variance of the $y$ values, was $0.338 \pm 0.058$, while the normalized error for the Isotron algorithm was $0.526 \pm 0.175$.  In addition, we observed that $\Liso$ performed
better rather consistently across the folds - the difference between the normalized error of Isotron and $\Liso$ was $0.189 \pm 0.139$.

%Table \ref{tab:overfit}
%shows the mean squared errors (normalized by the variance in $y$ values) of
%LIsotron and Isotron in a 10-fold cross validation. The third column in the
%jfigure shows the mean and variance of the difference of test errors of Isotron
%and LIsotron across the 10 folds. LIsotron performs better than Isotron on all
%folds.

%\begin{table}[ht]
%\caption{Performance comparison between Isotron and LIsotron on synthetic
%example (Section \ref{sec:expt-synthetic}). The columns corresponding to Isotron
%and LIsotron show the mean squared error and deviation, while the last column
%($\Delta$) shows the mean of the difference between the squared errors of the
%two algorithms and its deviation, across the 10 folds.\label{tab:overfit}}
%\begin{center}
%\begin{small}
%\begin{tabular}{ccc}
%Isotron & LIsotron & $\Delta$ \\
%\hline
%0.338 $\pm$ 0.0577 & 0.5266 $\pm$ 0.1752 & 0.1886 $\pm$ 0.1392 \\
%\hline
%\end{tabular}
%\end{small}
%\end{center}
%\end{table}

%\subsection{Exponential Link Function}

%The second synthetic experiment is to demonstrate the advantage of LIsotron over
%GLMs. When fitting data using GLMs it is typically necessary to try several
%different link functions to see which one performs better. Theoretically, we
%know that LIsotron can fit any link function. To test this experimentally we ran
%LIsotron and Logistic Regression on synthetic data generated from an exponential
%link function.

\begin{table*}[!ht]
\caption{Mean squared error normalized by the variance (mean and standard deviation across 10 folds).
\label{tab:real-world}}
\begin{center}
\begin{small}
\begin{tabular}{ccccccc}
% dataset & L-Iso & GLM & Iso & (K)L-Iso & (K)Iso & Lin-R & Log-R & Poi-R \\
% \hline
% communities & 0.34 $\pm$ 0.04 & 0.32 $\pm$ 0.03 & 0.34 $\pm$ 0.05 & 0.35 $\pm$
% 0.05 & 0.35 $\pm$ 0.03 & 0.34 $\pm$ 0.03 & 0.39 $\pm$ 0.05 \\
% concrete & 0.36 $\pm$ 0.06 & 0.33 $\pm$ 0.05 & 0.12 $\pm$ 0.03 & 0.13 $\pm$ 0.03
% & 0.39 $\pm$ 0.08 & 0.39 $\pm$ 0.04 & 0.44 $\pm$ 0.09 \\
% housing & 0.28 $\pm$ 0.13 & 0.18 $\pm$ 0.07 & 0.19 $\pm$ 0.08 & 0.24
% $\pm$ 0.12 & 0.28 $\pm$ 0.12 &0.27 $\pm$ 0.11 & 0.22 $\pm$ 0.09 \\
% parkinsons & 0.89 $\pm$ 0.04 & 0.88 $\pm$ 0.04 & 0.85 $\pm$ 0.04 & 0.85 $\pm$
% 0.04 & 0.90 $\pm$ 0.04 & 0.90 $\pm$ 0.04 & 0.92 $\pm$ 0.05 \\
% winequality & 0.77 $\pm$ 0.07 & 0.76 $\pm$ 0.07 & 0.73 $\pm$ 0.06 & 0.73 $\pm$
% 0.06 & 0.72 $\pm$ 0.08 & 0.72 $\pm$ 0.08 & 0.72 $\pm$ 0.72 \\
dataset & L-Iso & GLM-t & Iso &  Lin-R & Log-R & SIM \\
\hline
communities & 0.34 $\pm$ 0.04 & 0.34 $\pm$ 0.03 & 0.35 $\pm$ 0.04 & 0.35 $\pm$
0.04 & 0.34 $\pm$ 0.03 & 0.36 $\pm$ 0.05 \\
concrete & 0.35 $\pm$ 0.06 & 0.40 $\pm$ 0.07 & 0.36 $\pm$ 0.06 & 0.39 $\pm$ 0.08
& 0.39 $\pm$ 0.08 & 0.35 $\pm$ 0.06 \\
housing & 0.27 $\pm$ 0.12 & 0.28 $\pm$ 0.11 & 0.27 $\pm$ 0.12 & 0.28 $\pm$ 0.12
&0.27 $\pm$ 0.11 & 0.26 $\pm$ 0.09 \\
parkinsons & 0.89 $\pm$ 0.04 & 0.92 $\pm$ 0.04 & 0.89 $\pm$ 0.04 & 0.90 $\pm$ 0.04 & 0.90 $\pm$
0.04 & 0.92 $\pm$ 0.03 \\
winequality & 0.78 $\pm$ 0.07 & 0.81 $\pm$ 0.07 & 0.78 $\pm$ 0.07 & 0.73 $\pm$ 0.08 & 0.73 $\pm$ 0.08 & 0.79 $\pm$ 0.07 \\
\hline
\end{tabular}
\end{small}
\end{center}
\end{table*}

\begin{table*}[!ht]
\caption{Performance comparison of $\Liso$ with the other algorithms. The values reported are the difference in the normalized squared errors (mean and standard deviation across the 10 folds).
Negative values indicate better performance than $\Liso$.
\label{tab:real-world2}}
\begin{center}
\begin{small}
\begin{tabular}{cccccc}
% dataset & Iso & (K)L-Iso & (K)Iso & Lin-R & Log-R & Poi-R \\
% \hline
% communities & -0.01 $\pm$ 0.01 & 0.01 $\pm$ 0.02 & 0.01 $\pm$
% 0.02 & 0.02 $\pm$ 0.02 & 0.01 $\pm$ 0.02 & 0.06 $\pm$ 0.04 \\
% concrete & -0.02 $\pm$ 0.01 & -0.23 $\pm$ 0.06 & -0.23 $\pm$ 0.03 & 0.04 $\pm$ 0.03
% & 0.04 $\pm$ 0.03 & 0.09 $\pm$ 0.04  \\
% housing & -0.09 $\pm$ 0.09 & -0.08 $\pm$ 0.09 & -0.03 $\pm$ 0.07
% & 0.02 $\pm$ 0.05 &0.01 $\pm$ 0.05 & -0.05 $\pm$ 0.05 \\
% parkinsons & -0.01 $\pm$ 0 & -0.04 $\pm$ 0.02 & -0.03 $\pm$ 0.02 & 0.02 $\pm$
% 0.01 & 0.02 $\pm$ 0.01 & 0.03 $\pm$ 0.02 \\
% winequality & -0.01 $\pm$ 0.01 & -0.05 $\pm$ 0.03 & -0.04 $\pm$
% 0.04 & -0.05 $\pm$ 0.03 & -0.05 $\pm$ 0.03 & -0.05 $\pm$ 0.03 \\
% \hline
dataset & GLM-t & Iso & Lin-R & Log-R & SIM \\
\hline
communities & 0.01 $\pm$ 0.01 & 0.01 $\pm$ 0.02 & 0.02 $\pm$ 0.02 & 0.01 $\pm$
0.02 & 0.02 $\pm$
0.03 \\
concrete & 0.04 $\pm$ 0.03 & 0.00 $\pm$ 0.01 & 0.04 $\pm$ 0.03 & 0.04 $\pm$ 0.03
& 0.00 $\pm$ 0.02  \\
housing & 0.02 $\pm$ 0.05 & 0.00 $\pm$ 0.07 & 0.02 $\pm$ 0.05 &0.01 $\pm$ 0.05 & -0.01 $\pm$ 0.06
\\
parkinsons & 0.03 $\pm$ 0.02 & 0.00 $\pm$ 0.01 & 0.02 $\pm$ 0.01 & 0.02 $\pm$ 0.01 & 0.04 $\pm$ 0.04
\\
winequality & 0.03 $\pm$ 0.02 & 0.00 $\pm$ 0.01 & -0.05 $\pm$ 0.03 & -0.05 $\pm$ 0.03 & 0.01
$\pm$ 0.01 \\
\hline
\end{tabular}
\end{small}
\end{center}
\end{table*}

\begin{figure}[!ht]
\begin{center}
$\begin{array}{cc}
\includegraphics[width=0.4\columnwidth]{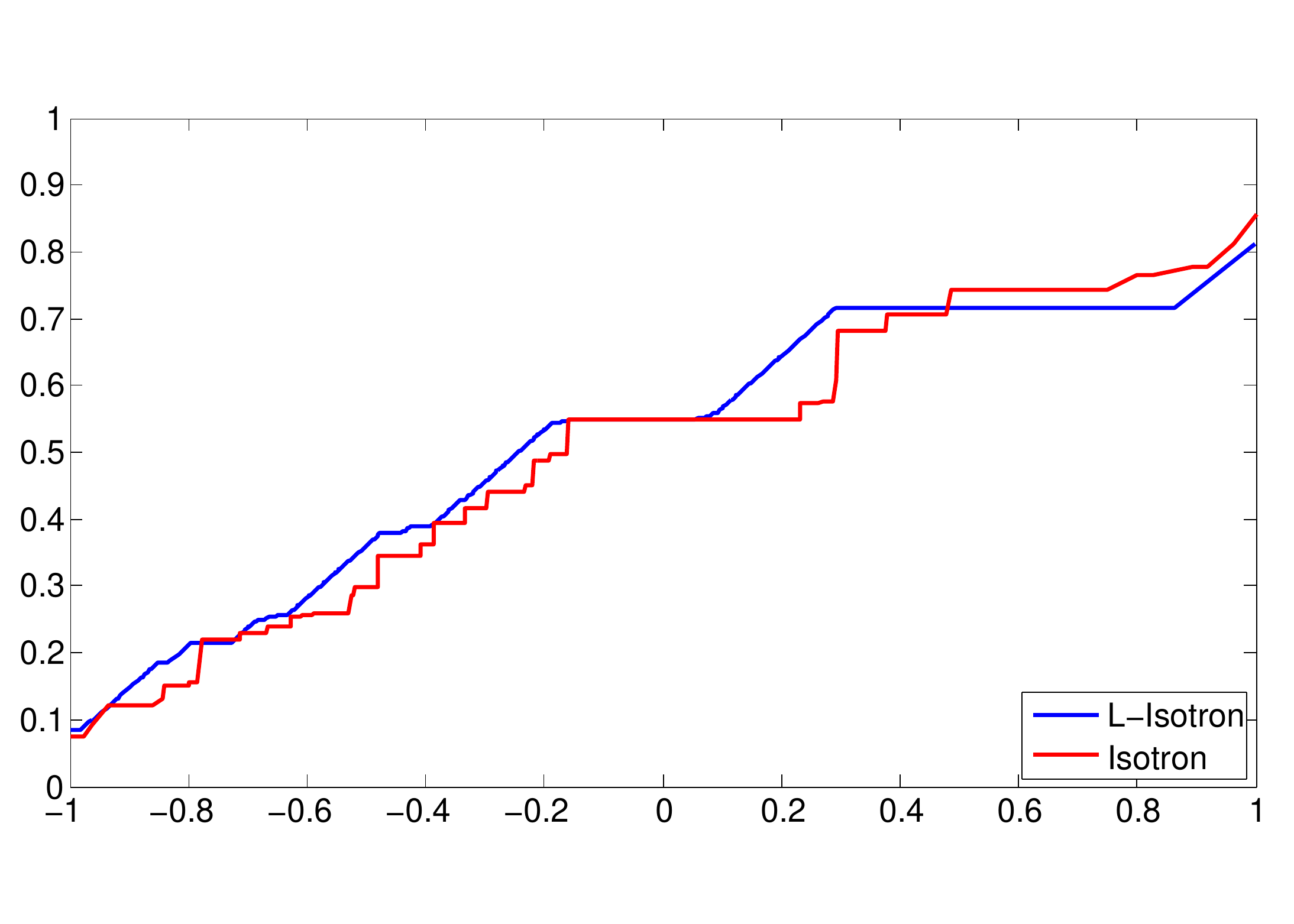} &
\includegraphics[width=0.4\columnwidth]{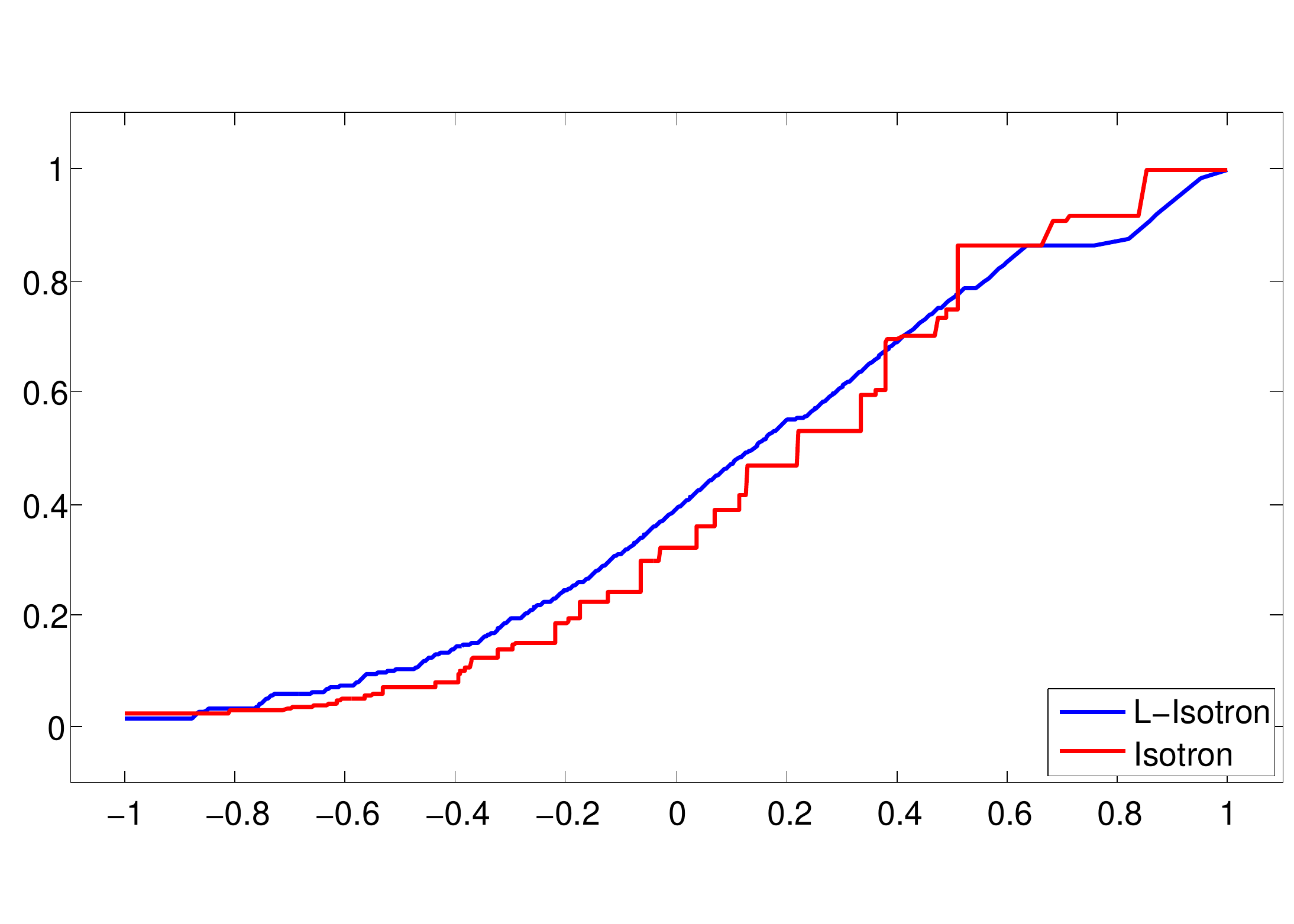}  \\
\mbox{(a) \texttt{concrete} } & \mbox{(b) \texttt{communities}}
\end{array}$
\end{center}
\caption{The transfer function $u$ as predicted by $\Liso$ (blue) and Isotron
(red) for the \texttt{concrete} and \texttt{communities} datasets. The domain of both functions was normalized to $[-1,1]$. \label{fig:overfit2}}
\end{figure}

\subsection{Real World Datasets}

%Isotron, standard regression techniques - logistic and linear, and a simple heuristic algorithm for Single Index Models (SIM). The goal of these set of experiments is to show that our algorithm

%The heuristic algorithm follows the lines of standard  We report results from experiments on 5 datasets from the UCI repository (\cite{UCI}): \texttt{communities}, \texttt{concrete}, \texttt{housing}, \texttt{parkinsons}, \texttt{wine-quality}.

We now turn to describe the results of experiments performed on several UCI
datasets. We chose the following 5 datasets: \texttt{communities}, \texttt{concrete}, \texttt{housing}, \texttt{parkinsons}, and \texttt{wine-quality}.

On each dataset, we compared the performance of $\Liso$ (L-Iso) and $\GLMt$
(GLM-t) with Isotron and several
other algorithms.
%experimented with $\GLMt$ (with the sigmoid function as the transfer function) and $\Liso$. For comparison, we also ran several other algorithms.
These include standard logistic regression (Log-R), linear regression (Lin-R)
and a
%, and the Isotron algorithm (using the same training set each time, as discussed in the previous subsection). In addition, we tested
simple heuristic algorithm (SIM) for single index models, along the lines of standard iterative maximum-likelihood procedures for these types of problems (e.g., \cite{Cosslett83}). The algorithm works by iteratively fixing the direction $w$ and finding the best transfer function $u$, and then fixing $u$ and optimizing $w$ via gradient descent. For each of the algorithms we performed 10-fold cross validation, using $1$ fold each time as the test set, and we report averaged results across the folds.

%\Liso$, Isotron and SIM are iterative algorithms. We ran each of them for $4000$ iterations using 8 folds as training data, 1 fold was used a hold-out set and the iteration at which the error on this set was least was used as the final hypothesis. The last fold was used as a test set. For logistic regression and linear regression, we used standard Matlab implementation, thus 9 folds were used for training and the last fold for testing.

Table \ref{tab:real-world} shows the mean squared error of all the algorithms
across ten folds normalized by the variance in the $y$ values. Table
\ref{tab:real-world2} shows the difference between squared errors between the
algorithms across the folds. The results indicate that the performance of
$\Liso$ and
$\GLMt$ (and even Isotron) is comparable to other regression techniques and in
many cases also slightly better. This suggests that these algorithms should work well in practice, while enjoying non-trivial theoretical guarantees.
%Although in theory $\Liso$ and Isotron do not have parameters, we
%believe that their performance might be tuned up using parameters such as the
%Lipschitz constant and the step-size at each iteration.

It is also illustrative to see how the transfer functions found by the two
algorithms, $\Liso$ and Isotron, compare to each other. In Figure
\ref{fig:overfit2}, we plot the transfer function for \texttt{concrete} and
\texttt{communities}.  The plots illustrate the fact that Isotron repeatedly fits a non-Lipschitz function resulting in a piecewise constant function, which is less intuitive than the smoother, Lipschitz transfer function found by the
$\Liso$ algorithm.

%We stress that matlab implementations and our implementation of SIM were
%not carefully fine-tuned, but neither were our implementations of $\Liso$ and
%Isotron.

%Isotron, kernelized versions of
%LIsotron and Isotron (Gaussian kernel), linear regression, and GLMs - logistic
%and poisson regression. We performed experiments on the following UCI datasets
%\cite{UCI} \texttt{housing}, \texttt{concrete}, \texttt{wine-quality},
%\texttt{community}, \texttt{forest-fires} and \texttt{parkinsons}. For each of
%the methods we performed 10 fold cross validation. Since LIsotron and Isotron
%are iterative algorithms, 8 folds were used for training, 9$^{\mbox{\tiny{th}}}$
%as a hold-out set and the 10$^{\mbox{\tiny{th}}}$ as test set. The values
%reported are mean squared error on the test set normalized by the variance in $y$
%values, along with the deviation across folds. Table \ref{tab:real-world}
%shows the performance of the various algorithms.
%
%
%\section{Conclusions}
%\label{sec:concl}

%% file: appendix.tex
\subsection{Proof of \thmref{thm:known-u}}
\label{app:glm-algo}

The reader is referred to $\GLMt$ (Alg. \ref{alg:fixed-u-alg}) for notation used in this section.

The main lemma shows that as long as the error of
the current hypothesis is large the distance of our predicted direction vector
$w^t$ from the ideal direction $w$ decreases.

\begin{lemma}
\label{lem:update-bounds1}
At iteration $t$ in $\GLMt$, suppose $\lnorm{w^t - w} \leq W$, then if $\lnorm
{(1/m) \sum_{i=1}^m (y_i - u(w \cdot x_i)) x_i }_w \leq \eta$, then
\[
\lnorm{w^t - w}^2 - \lnorm{w^{t+1} - w}^2 \geq \heps(h^t) - 5 W \eta
\]
\end{lemma}
%probability at least $ $m \geq 400 \log(1/\delta)/\epsilon^2$, then simultaneously for all $t$,
%$\lnorm{w^t - w}^2 - \lnorm{w^{t+1} - w}^2 \leq \heps(h^t) - \epsilon/2$
\begin{proof}
%Recall that by defintion $u(w \cdot x) = \E[y | x]$. Thus $y_i - u(w \cdot x_i)$
%are mean $0$ bounded random variables. Thus for the value of $m$, using Lemma
%\ref{lem:conc-bounds} we have
%\begin{align}
%\lnorm{\fraconem \sum_{i=1}^m (y_i - u(w \cdot x_i) x_i } \leq
%\frac{\epsilon}{8W}\label{eq:lem-known-u-conc-bound}
%\end{align}
We have
\begin{equation}
\lnorm{w^t - w}^2 - \lnorm{w^{t+1} - w}^2
~=~  \frac{2}{m} \sumionetom (y_i - u(w^t \cdot x_i))( w \cdot x_i - w^t \cdot x_i)
- \lnorm{\fraconem \sumionetom (y_i - u(w^t \cdot x_i)) x_i}^2.
\label{eq:lem-known-u-main-eqn}
\end{equation}
Consider the first term above,
\[
\frac{2}{m} \sumionetom (y_i - u(w^t \cdot x_i)) (w \cdot x_i - w^t \cdot x_i)
\nonumber \\
~=~ \frac{2}{m} \sumionetom (u(w \cdot x_i) - u(w^t \cdot x_i))(w \cdot x_i -
w^t \cdot x_i) + \frac{2}{m} \left(\sumionetom (y_i - u(w \cdot x_i)) x_i \right)
\cdot (w - w^t).
\]
Using the fact that $u$ is non-decreasing and $1$-Lipschitz (for the
first term) and $\lnorm{w - w^t} \leq W$ and $\lnorm{(1/m) \sum_{i=1}^m (y_i - u(w\cdot x_i)) x_i} \leq \eta$, we can lower bound this by
\begin{equation}
\frac{2}{m} \sumionetom (u(w \cdot x_i) - u(w^t \cdot x_i))^2 - 2
W\eta
~\geq~ 2 \heps(h^t) - 2W \eta. \label{eq:lem-known-u-first-term}
\end{equation}
For the second term in (\ref{eq:lem-known-u-main-eqn}), we have
\begin{align}
&\lnorm{\fraconem\sumionetom (y_i - u(w^t \cdot x_i))x_i }^2
~=~ \lnorm{\fraconem\sumionetom(y_i - u(w \cdot x_i) + u(w \cdot x_i) - u(w^t
\cdot x_i))x_i }^2 \notag \\
&~~\leq \lnorm{\frac{1}{m} \sumionetom (y_i - u(w \cdot x_i)) x_i}^2
~+~ 2 \lnorm{\frac{1}{m}\sumionetom(y_i - u(w \cdot x_i)) x_i} ~\times~
\lnorm{\fraconem \sumionetom (u(w \cdot x_i) - u(w^t \cdot x_i)) x_i}
\notag \\
&~~~~~~ + \lnorm{\frac{1}{m} \sumionetom (u(w \cdot x_i) - u(w^t \cdot x_i))
x_i}^2 \label{eq:manipulation}
\end{align}
Using the fact that $\lnorm{ (1/m)\sumionetom (y_i -u(w \cdot
x_i))x_i} \leq \eta$, and using Jensen's inequality to show that \\ $\lnorm{(1/m) \sumionetom (u(w \cdot x_i) - u(w^t \cdot
x_i)) x_i}^2 \leq (1/m)\sumionetom (u(w \cdot x_i) - u(w^t \cdot
x_i))^2=\heps(h^t)$, and assuming $W \geq 1$, we get
\begin{equation}
\lnorm{\fraconem \sumionetom (y_i - u(w \cdot x_i))x_i}^2 \leq \heps(h^t) + 3 W \eta \label{eq:lem-known-u-second-term}
\end{equation}
Combining (\ref{eq:lem-known-u-first-term}) and
(\ref{eq:lem-known-u-second-term}) in (\ref{eq:lem-known-u-main-eqn}), we get
\[
\lnorm{w^t - w}^2 - \lnorm{w^{t+1} - w}^2 \geq \heps(h^t) - 5 W \eta
\]
\end{proof}

The bound on $\heps(h^t)$ for some $t$ now follows from Lemma
\ref{lem:update-bounds1}. Let $\eta = 2(1+ \sqrt{log(1/\delta)/1})/\sqrt{m}$. Notice that $(y_i - u(w \cdot
x_i))x_i$ for all $i$ are i.i.d. $0$-mean random variables with norm bounded by $1$, so using Lemma \ref{lem:conc-bounds}, $\lnorm{(1/m)\sumionetom (y_i - u(w \cdot x_i)) x_i} \leq \eta$. Now using Lemma \ref{lem:update-bounds1}, at each iteration of algorithm $\GLMt$, either $\lnorm{w^{t+1} - w}^2 \leq \lnorm{w^t - w}^2 - W \eta$, or $\heps(h^t) \leq 6 W \eta$. If the latter is the case, we are done. If
not, since $\lnorm{w^{t+1} - w}^2 \geq 0$, and $\lnorm{w^0 - w}^2 =
\lnorm{w}^2 \leq W^2$, there can be at most $W^2/(W\eta) = W/(\eta)$
iterations before $\heps(h^t) \leq 6 W \eta$. Overall, there is some $h^t$ such that
\[
\heps(h^t) \leq O\left(\sqrt{\frac{W^2\log(1/\delta)}{m}}\right).
\]
In addition, we can reduce this to a high-probability bound on $\eps(h^t)$ using \lemref{lem:unconv}, which is applicable since $\norm{w^{t}}\leq W$. Using a union bound, we get a bound which holds simultaneously for $\heps(h^t)$ and $\eps(h^t)$.

\section{Proof of Proposition \ref{prop:close-fit}}\label{sec:proof-prop-close-fit}
\label{app:close-fit}

To prove the proposition, we actually prove a more general result.
Define the function class
\[
\Ucal =\{u:[-W,W]\rightarrow [0,1]: u\text{ 1-Lipschitz}\}.
\]
and
\[
\Wcal = \{x\mapsto \inner{x,w}:w\in \reals^d, \norm{w}\leq W\},
\]
where $d$ is possibly infinite (for instance, if we are using kernels).

It is easy to see that the proposition follows from the following uniform convergence guarantee:
\begin{thm}\label{thm:noisebound}
With probability at least $1-\delta$, for any fixed $w\in \Wcal$, if we let
\[
\hat{u} = \arg\min_{u\in \Ucal}\frac{1}{m}\sum_{i=1}^{m}(u(w\cdot x_i)-y_i)^2,
\]
and define
\[
\tilde{u} = \arg\min_{u\in\Ucal}\frac{1}{m}\sum_{i=1}^{m}(u(w \cdot x_i)-\E[y|x_i])^2,
\]
then
\[
\frac{1}{m}\sum_{i=1}^{m}|\hat{u}(w \cdot x_i)-\tilde{u}(w \cdot x_i)|
\leq O\left(\min\left\{\left(\frac{dW^{3/2}\log(Wm/\delta)}{m}\right)^{1/3}
+\sqrt{\frac{W^2\log(m/\delta)}{m}}~~,~~\left(\frac{W^2\log(m/\delta)}{m}\right)^{1/4}\right\}\right).
\]
\end{thm}

%We note that although the $m^{-1/3}$ rate might seem disappointing, it is typical for such $l_1$ distance bounds, using nonparametric function classes such as $\Ucal$ (see [Risk bounds for Isotonic Regression, Con-Hui Zhang 2002] in the context of monotone functions).

To prove the theorem, we use the concept of ($\infty$-norm) covering numbers. Given a function class $\Fcal$ on some domain and some $\epsilon>0$, we define $\Ncal_{\infty}(\epsilon,\Fcal)$ to be the smallest size of a \emph{covering set} $\Fcal'\subseteq \Fcal$, such that for any $f\in \Fcal$, there exists some $f'\in \Fcal$ for which $\sup_{x}|f(x)-f'(x)|\leq \epsilon$.
In addition, we use a more refined notion of an $\infty$-norm covering number, which deals with an empirical sample of size $m$. Formally, define $\Ncal_{\infty}(\epsilon,\Fcal,m)$ to be the smallest integer $n$, such that for any $x_1,\ldots,x_m$, one can construct a covering set $\Fcal'\subseteq \Fcal$ of size at most $n$, such that for any $f\in \Fcal$, there exists some $f'\in \Fcal$ such that $\max_{i=1,\ldots,m}|f(x_i)-f'(x_i)|\leq \epsilon$.

\begin{lemma}\label{lem:covbound}
Assuming $m,1/\epsilon,W\geq 1$, we have the following covering number bounds:
\begin{enumerate}
\item $\Ncal_{\infty}(\epsilon,\Ucal) \leq \frac{1}{\epsilon} 2^{2W/\epsilon}$.
\item $\Ncal_{\infty}(\epsilon,\Wcal) \leq  \left(1+\frac{2W}{\epsilon}\right)^d$.
\item $\Ncal_{\infty}(\epsilon,\Ucal\circ \Wcal) \leq  \frac{2}{\epsilon} 2^{4W/\epsilon}\left(1+\frac{4W}{\epsilon}\right)^d$.
\item $\Ncal_{\infty}(\epsilon,\Ucal\circ \Wcal,m) \leq \frac{2}{\epsilon} (2m+1)^{1+8W^2/\epsilon^2}$
\end{enumerate}
\end{lemma}

\begin{proof}
We start with the first bound. Discretize $[-W,W]\times [0,1]$ to a
two-dimensional grid \\ $\{-W+\epsilon a, \epsilon b\}_{a=0,\ldots,2W/\epsilon, b=0,\ldots,1/\epsilon}$. It is easily verified that for any function $u\in \Ucal$, we can define a piecewise linear function $u'$, which passes through points in the grid, and in between the points, is either constant or linear with slope 1, and $\sup_{x}|u(x)-u'(x)|\leq \epsilon$. Moreover, all such functions are parameterized by their value at $-W$, and whether they are sloping up or constant at any grid interval afterwards. Thus, their number can be coarsely upper bounded as $2^{2W/\epsilon}/\epsilon$.

The second bound in the lemma is a well known fact - see for instance pg. 63 in \cite{Pisier99}).

The third bound in the lemma follows from combining the first two bounds, and using the Lipschitz property of $u$ (we simply combine the two covers at an $\epsilon/2$ scale, which leads to a cover at scale $\epsilon$ for $\Ucal \circ \Wcal$).

To get the fourth bound, we note that by corollary 3 in \cite{Zhang02}. $\Ncal_{\infty}(\epsilon,\Wcal,m)\leq (2m+1)^{1+W^2/\epsilon^2}$. Note that unlike the second bound in the lemma, this bound is dimension-free, but has worse dependence on $W$ and $\epsilon$. Also, we have $\Ncal_{\infty}(\epsilon,\Ucal,m)\leq \Ncal_{\infty}(\epsilon,\Ucal)\leq \frac{1}{\epsilon}2^{2W/\epsilon}$ by definition of covering numbers and the first bound in the lemma. Combining these two bounds, and using the Lipschitz property of $u$, we get
\[
\frac{2}{\epsilon} (2m+1)^{1+4W^2/\epsilon^2}2^{4W/\epsilon}.
\]
Upper bounding $2^{4W/\epsilon}$ by $(2m+1)^{4W^2/\epsilon^2}$, the
the fourth bound in the lemma follows.
\end{proof}

\begin{lemma}\label{lem:dimfreebound}
With probability at least $1-\delta$ over a sample $(x_1,y_1),\ldots,(x_m,y_m)$ the following bounds hold simultaneously for any $w\in \Wcal,u,u'\in \Ucal$,
\[
\left|\frac{1}{m}\sum_{i=1}^{m} (u(w \cdot x_i)-y_i)^2-\E\left[(u(w \cdot x)-y)^2\right]\right| \leq O\left(\sqrt{\frac{W^2\log(m/\delta)}{m}}\right),
\]
\[
\left|\frac{1}{m}\sum_{i=1}^{m} (u(w \cdot x_i)-\E[y|x_i])^2-\E\left[(u(w \cdot x)-\E[y|x])^2\right]\right| \leq O\left(\sqrt{\frac{W^2\log(m/\delta)}{m}}\right),
\]
\[
\left|\frac{1}{m}\sum_{i=1}^{m}|u(w \cdot x_i)-u'(w \cdot x_i)|
-\E\left[|u(w \cdot x)-u'(w \cdot x)|\right]\right| \leq O\left(\sqrt{\frac{W^2\log(m/\delta)}{m}}\right)
\]
\end{lemma}
\begin{proof}
\lemref{lem:covbound} tells us that
$\Ncal_{\infty}(\epsilon,\Ucal\circ \Wcal,m) \leq \frac{2}{\epsilon} (2m+1)^{1+8W^2/\epsilon^2}$.
It is easy to verify that the same covering number bound holds for the function classes $\{(x,y)\mapsto (u(w \cdot x)-y)^2:u\in \Ucal,w\in \Wcal\}$ and $\{x\mapsto (u(w \cdot x)-\E[y|x])^2:u\in \Ucal,w\in \Wcal\}$, by definition of the covering number and since the loss function is $1$-Lipschitz. In a similar manner, one can show that the covering number of the function class $\{x\mapsto |u(w \cdot x)-u'(w \cdot x)|:u,u'\in \Ucal,w\in \Wcal\}$ is at most $\frac{4}{\epsilon} (2m+1)^{1+32W^2/\epsilon^2}$.

Now, one just need to use results from the literature which provides uniform convergence bounds given a covering number on the function class. In particular, combining a uniform convergence bound in terms of the Rademacher complexity of the function class (e.g. Theorem 8 in \cite{BartMen02}), and a bound on the Rademacher complexity in terms of the covering number, using an entropy integral (e.g., Lemma A.3 in \cite{SreSriTe10}), gives the desired result.
\end{proof}

\begin{lemma}\label{lem:compositebound}
With probability at least $1-\delta$ over a sample $(x_1,y_1),\ldots,(x_m,y_m)$, the following holds simultaneously for
any $w\in \Wcal$: if we let
\[
\hat{u}_{w}(\inner{w,\cdot})=\arg\min_{u\in \Ucal}\frac{1}{m}\sum_{i=1}^{m}(u(w \cdot x_i)-y_i)^2
\]
denote the empirical risk minimizer with that fixed $w$, then
\[
\E(\hat{u}_{w}(w \cdot x)-y)^2-\inf_{u\in \Ucal} \E(u(w \cdot x)-y)^2 \leq O\left(W\left(\frac{d\log(Wm/\delta)}{m}\right)^{2/3}\right),
\]
\end{lemma}

\begin{proof}
For generic losses and function classes, standard bounds on the the excess error typically scale as $O(1/\sqrt{m})$. However, we can utilize the fact that we are dealing with the squared loss to get better rates. In particular, using Theorem 4.2 in \cite{Mendel02}, as well as the bound on $\Ncal_{\infty}(\epsilon,\Ucal)$ from \lemref{lem:covbound}, we get that for any fixed $w$, with probability at least $1-\delta$,
\[
\E(\hat{u}_{w}(w \cdot x)-y)^2-\inf_{u\in \Ucal} \E(u(w \cdot x)-y)^2 \leq O\left(W\left(\frac{\log(1/\delta)}{m}\right)^{2/3}\right).
\]
To get a statement which holds simultaneously for any $w$, we apply a union bound over a covering set of $\Wcal$. In particular, by \lemref{lem:covbound}, we know that we can cover $\Wcal$ by a set $\Wcal'$ of size at most $(1+2W/\epsilon)^d$, such that any element in $\Wcal$ is at most $\epsilon$-far (in an $\infty$-norm sense) from some $w'\in \Wcal'$. So applying a union bound over $\Wcal'$, we get that with probability at least $1-\delta$, it holds simultaneously for any $w'\in \Wcal$ that
\begin{equation}\label{eq:ww}
\E(\hat{u}_{w'}(\inner{w',x})-y)^2-\inf_{u} \E(u(\inner{w',x})-y)^2 \leq O\left(W\left(\frac{\log(1/\delta)+d\log(1+2W/\epsilon)}{m}\right)^{2/3}\right).
\end{equation}
Now, for any $w\in \Wcal$, if we let $w'$ denote the closest element in $\Wcal'$, then $u(w \cdot x)$ and $u(\inner{w',x})$ are $\epsilon$-close \emph{uniformly} for any $u\in \Ucal$ and any $x$. From this, it is easy to see that we can extend \eqref{eq:ww} to hold for any $\Wcal$, with an additional $O(\epsilon)$ element in the right hand side. In other words, with probability at least $1-\delta$, it holds simultaneously for any $w\in \Wcal$ that
\[
\E(\hat{u}_{w}(w \cdot x)-y)^2-\inf_{u} \E(u(w \cdot x)-y)^2 \leq O\left(W\left(\frac{\log(2/\delta)+d\log(1+2W/\epsilon)}{m}\right)^{2/3}\right)+\epsilon.
\]
Picking (say) $\epsilon=1/m$ provides the required result.
\end{proof}

\begin{lemma}\label{lem:lossdev}
Let $F$ be a convex class of functions, and let $f^{*}=\arg\min_{f\in F} \E[(f(x)-y)^2]$. Suppose that $\E[y|x]\in \Ucal\circ\Wcal$. Then for any $f\in F$, it holds that
\[
\E[(f(x)-y)^2]-\E[(f^{*}(x)-y)^2] \geq \E\left[\left(f(x)-f^{*}(x)\right)^2\right] \geq \left(\E\left[\left|f(x)-f^{*}(x)\right|\right]\right)^2.
\]
\end{lemma}
\begin{proof}
It is easily verified that
\begin{equation}\label{eq:equiv}
\E[(f(x)-y)^2]-\E[(f^{*}(x)-y)^2] =
 \E_{x}[(f(x)-\E[y|x])^2-(f^{*}(x)-\E[y|x])^2].
\end{equation}
This implies that $f^{*}=\arg\min_{f\in F} \E[(f(x)-\E[y|x])^2]$.

Consider the $L_2$ Hilbert space of square-integrable functions, with respect to the measure induced by the distribution on $x$ (i.e., the inner product is defined as $\inner{f,f'}=\E_{x}[f(x)f'(x)]$). Note that $\E[y|x]\in \Ucal\circ \Wcal$ is a member of that space. Viewing $\E[y|x]$ as a function $y(x)$, what we need to show is that
\[
\norm{f-y}^2-\norm{f^{*}-y}^2 \geq \norm{f-f^{*}}^2.
\]
By expanding, it can be verified that this is equivalent to showing
\[
\inner{f^{*}-y,f-f^{*}}\geq 0.
\]
To prove this, we start by noticing that according to \eqref{eq:equiv}, $f^{*}$ minimizes $\norm{f-y}^2$ over $F$. Therefore, for any $f\in F$ and any $\epsilon\in (0,1)$,
\begin{equation}\label{eq:inprod}
\norm{(1-\epsilon)f^{*}+\epsilon f-y}^2-\norm{f^{*}-y}^2\geq 0,
\end{equation}
as $(1-\epsilon)f^{*}+\epsilon f \in F$ by convexity of $F$.
However, the right hand side of \eqref{eq:inprod} equals
\[
\epsilon^2\norm{f-f^{*}}^2+2\epsilon\inner{f^{*}-y,f-f^{*}},
\]
so to ensure \eqref{eq:inprod} is positive for any $\epsilon$, we must have
$\inner{f^{*}-y,f-f^{*}}\geq 0$. This gives us the required result, and establishes the first inequality in the lemma statement. The second inequality is just by convexity of the squared function.
\end{proof}

\begin{proof}[Proof of \thmref{thm:noisebound}]
We bound $\frac{1}{m}\sum_{i=1}^{m}|\hat{u}(w \cdot x)-\tilde{u}(w \cdot x)|$ in two different ways, one which is dimension-dependent and one which is dimension independent.

We begin with the dimension-dependent bound. For any fixed $w$, let $u^{*}$ be $\arg\min_{u\in \Ucal} \E(u(w \cdot x)-y)^2$. We have from \lemref{lem:compositebound} that with probability at least $1-\delta$, simultaneously for all $w\in \Wcal$,
\[
\E(\hat{u}(w \cdot x)-y)^2- \E(u^{*}(w \cdot x)-y)^2 \leq O\left(W\left(\frac{d\log(Wm/\delta)}{m}\right)^{2/3}\right),
\]
and by \lemref{lem:lossdev}, this implies
\begin{equation}\label{eq:hat}
\E[|\hat{u}(w \cdot x)-u^{*}(w \cdot x)|] \leq O\left(\left(\frac{dW^{3/2}\log(Wm/\delta)}{m}\right)^{1/3}\right).
\end{equation}
Now, we note that since $u^{*}=\arg\min_{u\in \Ucal} \E(u(w \cdot x)-y)^2$, then
$u^{*}=\arg\min_{u\in \Ucal} \E(u(w \cdot x)-\E[y|x])^2$ as well. Again applying \lemref{lem:compositebound} and \lemref{lem:lossdev} in a similar manner, but now with respect to $\tilde{u}$, we get that with probability at least $1-\delta$, simultaneously for all $w\in \Wcal$,
\begin{equation}\label{eq:tilde}
\E[|\tilde{u}(w \cdot x)-u^{*}(w \cdot x)|] \leq O\left(\left(\frac{dW^{3/2}\log(Wm/\delta)}{m}\right)^{1/3}\right).
\end{equation}
Combining \eqref{eq:hat} and \eqref{eq:tilde}, with a union bound, we have
\[
\E[|\hat{u}(w \cdot x)-\tilde{u}(w \cdot x)|] \leq O\left(\left(\frac{dW^{3/2}\log(Wm/\delta)}{m}\right)^{1/3}\right).
\]
Finally, we invoke the last inequality in \lemref{lem:dimfreebound}, using a union bound, to get
\[
\frac{1}{m}\sum_{i=1}^{m}|\hat{u}(w \cdot x)-\tilde{u}(w \cdot x)| \leq O\left(\left(\frac{dW^{3/2}\log(Wm/\delta)}{m}\right)^{1/3}
+\sqrt{\frac{W^2\log(m/\delta)}{m}}\right).
\]

We now turn to the dimension-independent bound. In this case, the covering number bounds are different, and we do not know how to prove an analogue to \lemref{lem:compositebound} (with rate faster than $O(1/\sqrt{m})$). This leads to a somewhat worse bound in terms of the dependence on $m$.

As before, for any fixed $w$, we let $u^{*}$ be $\arg\min_{u\in \Ucal} \E[(u(w \cdot x)-y)^2]$. \lemref{lem:dimfreebound} tells us that the empirical risk $\frac{1}{m}\sum_{i=1}^{m}(u(w \cdot x_i)-y_i)^2$ is concentrated around its expectation uniformly for any $u,w$. In particular,
\[
\left|\frac{1}{m}\sum_{i=1}^{m} (\hat{u}(w \cdot x_i)-y_i)^2-\E\left[\hat{u}(w \cdot x)-y)^2\right]\right| \leq O\left(\sqrt{\frac{W^2\log(m/\delta)}{m}}\right)
\]
as well as
\[
\left|\frac{1}{m}\sum_{i=1}^{m} (u^{*}(w \cdot x_i)-y_i)^2-\E\left[(u^{*}(w \cdot x)-y)^2\right]\right| \leq O\left(\sqrt{\frac{W^2\log(m/\delta)}{m}}\right),
\]
but since $\hat{u}$ was chosen to be the empirical risk minimizer, it follows that
\[
\E\left[(\hat{u}(w \cdot x_i)-y_i)^2\right]
-\E\left[(u^{*}(w \cdot x)-y)^2\right] \leq O\left(\sqrt{\frac{W^2\log(m/\delta)}{m}}\right),
\]
so by \lemref{lem:lossdev},
\begin{equation}\label{eq:actual}
\E\left[|u^{*}(w \cdot x)-\hat{u}(w \cdot x)|\right]\leq O\left(\left(\frac{W^2\log(m/\delta)}{m}\right)^{1/4}\right)
\end{equation}
Now, it is not hard to see that if $u^{*}=\arg\min_{u\in \Ucal} \E[(u(w \cdot x)-y)^2]$, then $u^{*}=\arg\min_{u\in \Ucal} \E[(u(w \cdot x)-\E[y|x])^2]$ as well. Again invoking \lemref{lem:dimfreebound}, and making similar arguments, it follows that
\begin{equation}\label{eq:expected}
\E\left[|u^{*}(w \cdot x)-\tilde{u}(w \cdot x)|\right]\leq O\left(\left(\frac{W^2\log(m/\delta)}{m}\right)^{1/4}\right).
\end{equation}
Combining \eqref{eq:actual} and \eqref{eq:expected}, we get
\[
\E\left[|\hat{u}(w \cdot x)-\tilde{u}(w \cdot x)|\right]\leq O\left(\left(\frac{W^2\log(m/\delta)}{m}\right)^{1/4}\right).
\]
We now invoke \lemref{lem:dimfreebound} to get
\begin{equation}\label{eq:dimfreebound}
\frac{1}{m}\sum_{i=1}^{m}|\hat{u}(w \cdot x_i)-\tilde{u}(w \cdot x_i)|
\leq O\left(\left(\frac{W^2\log(m/\delta)}{m}\right)^{1/4}\right).
\end{equation}
\end{proof}